\documentclass{amsart}
\usepackage{amssymb,amsfonts}
\usepackage{enumerate}
\usepackage{mathrsfs}
\usepackage{tikz}
\usepackage{tikz-cd}
\usepackage{algorithm}
\usepackage{algorithmic}
\usepackage{placeins}
\usepackage{hyperref}
\usepackage[maxbibnames=99]{biblatex}
\usepackage{caption}
\usepackage{subcaption}
\usepackage[textsize=small]{todonotes}

\newtheorem{thm}{Theorem}[section]
\newtheorem*{thm*}{Theorem}

\newtheorem{prop}[thm]{Proposition}
\newtheorem{lem}[thm]{Lemma}

\theoremstyle{definition}
\newtheorem{defn}[thm]{Definition}

\theoremstyle{remark}
\newtheorem{rem}[thm]{Remark}

\let\stopcode\relax
\long\def\startcode#1\stopcode{\unskip
\begingroup
\advance\leftskip by \parindent\parindent=-\parindent\advance\parskip by1ex
\par #1\endgroup}

\newcommand{\card}{{\texttt{\itshape\#}}}
\newcommand{\abs}[1]{\left\lvert #1 \right\rvert}
\newcommand{\norm}[1]{\left\lVert #1 \right\rVert}

\DeclareMathOperator{\circmean}{CircMean}

\newcommand{\ambdim}{N}

\newcommand{\bR}{\mathbb{R}}

\newcommand{\bZ}{\mathbb{Z}}

\makeatletter
\let\c@equation\c@thm
\makeatother
\numberwithin{equation}{section}

\title[Circular coordinates in recurrent time series]{Subsampling, aligning, and averaging to find circular coordinates in recurrent time series}

\author[A. J. Blumberg]{Andrew J. Blumberg}
\address{Irving Institute for Cancer Dynamics \\ Departments of Mathematics and Computer Science \\ Columbia University, NY}
\email{andrew.blumberg@columbia.edu}
\thanks{The first author was partially supported by the NSF grant
  DMS-1912194 and by ONR grant N00014-22-1-2679.}

\author{Mathieu Carri\`ere}
\address{DataShape, Centre Inria d'Universit\'e d'Azur \\ Biot, France}
\email{mathieu.carriere@inria.fr}

\author{Jun Hou Fung}
\address{Department of Systems Biology \\ Columbia University, NY}
\email{jf3380@cumc.columbia.edu}
\thanks{The third author was supported by the NSF grant DMS-1912194.}

\author[M. A. Mandell]{Michael A. Mandell}
\address{Department of Mathematics \\ Indiana University, IN}
\email{mmandell@iu.edu}
\thanks{The fourth author was supported by the ONR grant N00014-22-1-2675}

\begin{document}

\begin{abstract}
We introduce a new algorithm for finding
robust circular coordinates on data that is expected to exhibit
recurrence, such as that which appears in neuronal recordings of
\textit{Caenorhabditis elegans}.  Techniques exist to create circular coordinates on a simplicial complex from
a dimension 1 cohomology class, and these can be applied to the Rips
complex of a dataset when it has a prominent class in its dimension 1 cohomology.  However, it is known this approach is
extremely sensitive to uneven sampling density.

Our algorithm comes with a new method to correct for uneven
sampling density, adapting our prior
work on averaging coordinates in manifold learning.  We use rejection 
sampling to correct for inhomogeneous sampling and then apply 
Procrustes matching to align and average the
subsamples. In addition to providing a more robust coordinate than
other approaches, this subsampling and averaging approach has better efficiency.

We validate our technique on both synthetic data sets and
neuronal activity recordings.  Our results reveal a topological model of neuronal
trajectories for \textit{C. elegans} that is
constructed from loops in which different regions of the brain state
space can be mapped to specific and interpretable macroscopic
behaviors in the worm.
\end{abstract}

\maketitle

\tableofcontents

\section{Introduction}

In this paper, we consider coordinates on recurrent time series data.
Good coordinates should reflect this recurrence and so should take
values on a circle.  There has been a
substantial literature on producing coordinates of this type (e.g.,
\cite{Linderman2019, Brennan2023, ZYL24} in the context of neuronal activity).
The work of de Silva, Morozov, and
Vejdemo-Johannson~\cite{deSilva2011} introduced the idea of using
techniques from {\em topological data analysis} (TDA) to find circular
coordinates, specifically {\em persistent cohomology}.  Their basic approach
is to infer a cohomology class and choose a specific cocycle
representative that is the smoothest, which amounts to solving a
least-squares optimization problem.  
In~\cite{VejdemoJohansson2015}, these techniques were applied to
detect and quantify recurrent patterns in human motions.

TDA methods have qualitative differences from competing approaches
that make them particularly attractive for exploratory data analysis.
For example, when the feature scale is unknown, persistent cohomology 
can be used to find a range of scales where the coordinate appears the 
most robust.  However, as a consequence of the optimization scheme used to produce a
cohomology class representative, regions that are more densely sampled
tend to be represented by smaller changes in the coordinate.  
This issue was identified in~\cite{Rybakken2019} where a
renormalization technique was introduced to try to correct for uneven
sampling.  Similarly, \cite{Paik2023} proposed using weighted Laplacians or $L^p$-optimization to infer robust circular coordinates.  Another issue that arises is that the computational cost of
computing the persistent cocycle is prohibitive even for moderately
large data sets.  This was addressed in~\cite{Perea2020} by
aggressively choosing landmarks and using those to compute
coordinates. There, an adapted version of the weighting scheme
of~\cite{Rybakken2019} was used to handle uneven density and possible
sensitivity to the specific choice of landmarks.  On the other hand, the
use of landmarks is particularly sensitive to errors introduced by
noise and outliers.

We introduce a new algorithm for finding circular coordinates that does not require
an ad hoc weighting scheme, is insensitive to uneven sampling density, and is robust to noise and outliers, while retaining efficiency.  At a high level, our algorithm
is comprised of the following steps: 

\begin{enumerate}
\item {\em Rejection sampling.} We apply density-based rejection
sampling to the data to produce a number of subsample data sets that with
high probability correct for uneven sampling of the original data.
\item {\em Persistent cohomology.} We identify a long bar in the
first persistent cohomology group $PH^{1}$ on each subsample, which we
use to obtain a circular coordinate adapted to that subsample.
\item {\em Alignment and averaging.} We align and average the
resulting coordinates to obtain a final coordinate on the entire data set.
\end{enumerate}

We use standard density estimation techniques to carry out the
rejection sampling; these can be shown to be asymptotically correct
and work well in practice.  For the alignment and averaging, we build
on our prior work~\cite{BCFM} that uses the solution to the generalized
Procrustes problem to align and average different Euclidean
coordinates.  We use $O(2)$ Procrustes alignment on $\bR^{2}$ to
produce an initial seed for a hill climbing algorithm to search for
the best alignment of circular coordinates for averaging.   

We applied our rejection sampling, alignment, and averaging
algorithm to extract circular coordinates from whole-brain neuronal 
recordings from the nematode \textit{C.\ elegans}.  There is a 
recurrent locomotory cycle reflected in the worm's neuronal activity, which explains why circular coordinates are sensible.  We found coordinates
for a variety of worm neuronal trajectories, providing a basis for further scientific
analysis of stimulus response.  Furthermore, we validated our approach using an
information-theoretic criterion, showing that the resulting
coordinates are more 
informative than the uncorrected coordinates.  Finally, we observed
significant time savings from computing and aggregating persistent cocycle
representatives on subsamples.

\subsection*{Outline}
Section~\ref{sec:circ-coords-review} reviews the use of persistent
cohomology to find circular coordinates.  Section~\ref{sec:density}
describes the method for performing rejection sampling to construct
the subsamples for circular coordinates.  Section~\ref{sec:procrustes}
describes the algorithm to take the coordinates on the subsamples and
build a coordinate on the whole data set by alignment and averaging.
Section~\ref{sec:experimental} presents experimental results of our
algorithm applied to synthetic and real world data.

\subsection*{Acknowledgments}

We thank Ra\'{u}l Rabad\'{a}n and Eviatar Yemini for many
helpful conversations related to the subject of this paper.

\section{Finding circular coordinates using persistent cohomology}\label{sec:circ-coords-review}

In this section, we review techniques from \emph{topological data
analysis} (TDA) for extracting circular coordinates from
high-dimensional data, following work of~\cite{deSilva2011},
\cite{Rybakken2019}, and~\cite{Perea2020}.  

The basic overall process of generating \emph{circular coordinates} $X \to S^1$ from a data set $X$ is as follows:

\begin{enumerate}

\item We construct a filtered simplicial complex $K_\bullet$, the
  Vietoris-Rips complex, on the data set $X$.  The filtration encodes
  information about the shape of the data as a scale parameter $\epsilon$
  varies.

\item We compute the first persistent cohomology $PH^1$ of this
  filtered complex, and find a cohomology class which is present
  across a wide range of feature scales (the ``longest bar'').

\item We identify the value of the scale parameter $\epsilon$ within the
  longest bar, following \cite{Rybakken2019}.  

\item We perform an optimization process to extract a
  coordinate map $\bar f$ to $S^1$ from this persistent cohomology class.

\end{enumerate}

We assume the reader is familiar enough with TDA that the first two
steps do not require additional exposition.  (See for
example~\cite{Carlsson-Intro, Ghrist-Intro, Lesnick-Intro} for a basic
overview or~\cite{Edelsbrunner2010,  Oudot2015, Chazal2016,
Rabadan2019} for more detailed treatments.) We also point out that the theoretical aspects of ensemble approaches to persistent homology, similar to what we shall describe later, have been recently explored \cite{Solomon2022}.

In step 3, it may be the case that the scale parameter is not very
large.  There are many possible heuristics for what ``large'' means;
for example, comparison to ensure the scale parameter is substantially
larger than the median $k$th nearest neighbor length for $k=3$ is a common
criterion. In such a case, it is reasonable to infer that cohomology
classes in $PH^{1}$ are either too small to reflect actual geometry in
the data or the geometry underlying the point cloud cannot be
adequately described by a circular coordinate.  This ability to detect
cases when the assumed geometric model is misspecified is a strength
of the TDA approach.  We can optionally alert the user when the identified bar
is too small.  Moreover, if there are multiple large bars, we can notify the user that there may be multiple independent and
informative coordinates.

We now give a terse review of the procedure in step 4, based
on~\cite{deSilva2011}. 
A foundational result in algebraic topology is that for a simplicial
complex $K_\bullet$, elements of $H^{1}(K_\bullet;\bZ)$ are in
one-to-one correspondence with homotopy classes of maps from the
geometric realization $K$ to $S^{1}$.  There is a straightforward way
to go from an integral $1$-cocycle $\alpha$ to such a map: having
chosen a well-ordering of the vertices of $K_\bullet$, the data in
$\alpha$ (the cochain itself) is an integer-valued function on the
set of edges.  The map that sends all vertices of $K_\bullet$ to a
given basepoint of $S^{1}$ and sends each edge $e$ to $S^{1}$ by
wrapping around $\alpha(e)$ times (using the orientation from the ordering
of vertices) gives a map $K\to S^{1}$ in the corresponding homotopy
class.  In our setting, this construction is useless because when
$K_\bullet$ is the Vietoris-Rips complex of a data set, the resulting coordinate map $\bar f$ on the data
set is a constant map.

To address this issue, de Silva, Morozov, and Vejdemo-Johansson
\cite{deSilva2011} propose a method using real cocycles.  In the
following, $C^{n}(K_\bullet,A)$ denotes cochains with coefficients in $A$
(functions from $K_{n}$ to $A$), and $\delta$ represents the
differential in the cochain complex.

\begin{prop}[{\cite[Proposition 2]{deSilva2011}}]\label{prop:desilva}
Let $[\tilde{\alpha}] \in \operatorname{im}(H^1(K_\bullet; \bZ)
\to H^1(K_\bullet; \bR))$ and let $\tilde \alpha$ be a cocycle
representative of the form $\tilde{\alpha} = \alpha +
\delta f$ for some $\alpha \in C^1(K_\bullet; \bZ)$ and $f \in
C^0(K_\bullet; \bR)$.  Then there is a continuous function 
$K \to \bR/\bZ$ that sends each $x \in K_0$ to
$f(a) \pmod{\bZ}$ and each $e \in K_1$ to an interval of
(signed) length $\tilde{\alpha}(e)$.  
\end{prop}

In the case when $K_\bullet$ is the Vietoris-Rips complex of a data
set $X$ at a scale $\epsilon$, the set of vertices $K_{0}$ is $X$ and
we get coordinates $X\to \bR/\bZ$ given by the
function in the proposition.  Namely, $x\in X$ goes to the reduction
mod $\mathbb{Z}$ of $f(x)$.  To get circular coordinates $X\to
S^{1}$ with $S^{1}$ modeled as $\bR/2\pi\bZ$, we take $\bar f(x)=2\pi f(x)$.
For $S^{1}$, modeled as the unit circle in $\bR^{2}$, we take $\bar
f(x)=(\cos (2\pi f(x)),\sin(2\pi f(x)))$.

The procedure is then to start with an integral $1$-cocycle $\alpha$
and produce an appropriate real cocycle $\tilde{\alpha}$ that is
cohomologous to $\alpha$ in $C^1(K_\bullet; \bR)$.  The approach
of~\cite{deSilva2011} suggests using the ``harmonic''
representative, that is, the one that minimizes the energy functional
\[
E(\tilde \alpha)=\sum_{e\in K_{1}}\tilde\alpha(e)^{2}.
\]
Since we are solving the problem over the space $\tilde \alpha =\alpha
+\delta f$, if we rewrite in terms of $f$ we can implement this by solving
a least squares problem in the variables given by $K_{0}$.  This has a
unique solution provided every vertex lies on some edge. Otherwise,
isolated vertices give free parameters in the solution,
and this amounts to assigning them arbitrary circular coordinates.


\section{Density uniformization via rejection sampling}\label{sec:density}

The process described in the previous section is sensitive to the
density variation within the point cloud.  Non-uniform sampling causes distortion in
the circular coordinates, as we illustrate in
subsection~\ref{ss:density}. This section explains a general procedure to
fix this problem using rejection sampling.

The problem of non-uniform sampling distorting circular coordinates is
well-known and has been studied before. For example, Rybakken, Baas,
and Dunn~\cite{Rybakken2019}, Perea~\cite{Perea2020}, and Paik and Park~\cite{Paik2023} propose
additional steps to help ameliorate density imbalances when computing
circular coordinates.  These works suggest modifying the
optimization problem to find a ``weighted'' harmonic representative,
where the weights come from an altered metric on the underlying space
that accounts for density.  There are different possible
heuristics for choosing these weights.

Rather than attempting to infer weights based on density variations, we directly
equalize the density by rejection sampling and average the
resulting coordinates using Procrustes alignment of the coordinates
produced from subsamples.  This procedure has several notable advantages.
First, averaging produces more robust coordinates that are less
sensitive to outliers (e.g., see~\cite{BCFM} for a discussion of this in
the context of manifold learning).  Second, it is simple and easy
to generalize.  Third, we can use existing theoretical work on density
estimation to justify the subsampling procedure, whereas there does
not seem to be complete justification for any of the existing weighting
heuristics.
    
We begin by reviewing rejection sampling in the geometric context.
Let $M$ be a $d$-dimensional compact submanifold of $\bR^\ambdim$,
with inherited Riemannian metric $d_{M}$ and volume form $\mu$.  We assume that we only have
access to samples of $M$ from a different unknown density $\rho$.
Let $X$ be a sample of points in $M$ drawn with density $\rho$. 
Given a function $\pi\colon M\to [0,1]$, \emph{rejection sampling}
weighted by $\pi$ forms a subsample $Y$ from $X$ where for each point 
$x \in X$, we accept $x$ into $Y$ with probability $\pi(x)$.  If we choose
$\pi(x)$ proportional to $1/\rho(x)$, then $Y$ becomes distributed according to
$\mu$: by Bayes' theorem, for a Borel measurable subset $U \subseteq
M$, for each point $y$ in $Y$, the probability that $y$ belongs to $U$ is
\[
\frac{\int_U \pi(x) \rho(x) \, d\mu(x)}{\int_M \pi(x) \rho(x) \, d\mu(x)}=\frac{\mu(U)}{\mu(M)}.
\]

Of course, the density $\rho$ is not known a priori, so we have to
estimate it. The problem of density estimation is classical and
well-studied.  For instance, density estimators based on constructing
an empirical density function go back to Fix and Hodges
\cite{Fix1951}, Rosenblatt \cite{Rosenblatt1956}, and Parzen
\cite{Parzen1962}.  Estimators based on nearest neighbor distances
were originally proposed by Loftsgaarden and Quesenberry
\cite{Loftsgaarden1965}.  For textbook references on density
estimation, see \cite{Devroye1987, Scott2015}.  Recently, there has
been renewed interest in these types of density estimates and their
extensions in the context of geometric inference; see \cite{Biau2011,
Chazal2011a}.  Moreover, density estimation on unknown submanifolds
has also been specifically explored in~\cite{Berenfeld2021, Hickok}.
We do not claim particular originality for our discussion below; the
result we prove is very similar to existing results, slightly
generalized to fit our specific setup.

For simplicity, we work with a density estimator of the following form.  Fix
$\epsilon > 0$.  Let 
\begin{equation}\label{eq:densityestimator}
\hat{\rho}_\epsilon(x) = \frac{1}{V_d
  \epsilon^d} \frac{\card({X \cap B_\epsilon(x)})}{\card{X}},
\end{equation}
where $B_\epsilon(x)$ is a ball of radius $\epsilon$ centered at $x$ in
$\bR^\ambdim$ and $V_d$ is the volume of the $d$-dimensional unit ball.  A
different approach is to use density estimators based on codensity or
kernel methods; we found in practice that the basic density
estimator above was effective enough.

To establish the consistency of this estimator, we begin with an
elementary geometric lemma estimating the measure of a ball on the
manifold.  Recall that the \emph{reach}
of a smooth compact submanifold $M$ in $\bR^{\ambdim}$ is
the supremum of the real numbers $r$ such that every point in the
$r$-neighborhood of $M$ has a unique closest point to $M$.

\begin{lem}
Let $\epsilon$ be smaller than the reach of $M$.  Then for $x \in M$,
\[
\frac{\mu(M \cap B_\epsilon(x))}{V_d \epsilon^d} = 1 + O(\epsilon).
\] 
\end{lem}

\begin{proof}
We estimate the volume $\mu(M \cap B_\epsilon(x))$ as follows.
First, observe that there exists $\epsilon' = \epsilon + O(\epsilon^2)$ such 
that
\[
\mu(B^M_\epsilon(x)) \leq \mu(M \cap B_\epsilon(x)) \leq
\mu(B^M_{\epsilon'}(x)).
\]
The first inequality is immediate and the second follows from the fact that if the Euclidean distance from $y \in M$ to $x$
is less than twice the reach, then 
\[
d_M(x,y) \leq d_{\bR^n}(x,y) + O(d_{\bR^n}(x,y)^2).
\]
(One can give a direct proof by estimating arclength and using the
exponential map; see
e.g.,~\cite[Lemma~3]{Boissonnat2019},\cite[Lemma~4.3]{Belkin2008}, or
\cite[Lemma~3]{Bernstein2000}).  Therefore, $M \cap B_\epsilon(x)
\subseteq B^M_{\epsilon'}(x)$ for $\epsilon'$ large enough and on the
order of $\epsilon + O(\epsilon^2)$, which implies the inequality.

Next, we can estimate $\mu(B^M_\epsilon(x))$ in terms of
$\mu(B^{\bR^d}_\epsilon) = V_d \epsilon^d$ and higher corrections
coming from the scalar curvature $S$:
\[
\mu(B^M_\epsilon(x)) = V_d \epsilon^d \left(1 - \frac{S(x)}{6(d+2)}
\epsilon^2 + O(\epsilon^4)\right).
\]
As a consequence, we have that
\[
\mu(B^M_\epsilon(x)) = V_d \epsilon^d (1 + O(\epsilon)).
\]
Putting this all together, we conclude that
\[
\mu(M \cap B_\epsilon(x)) = V_d \epsilon^d (1 + O(\epsilon)).
\qedhere
\]
\end{proof}

Applying the lemma, we can now prove the consistency of the density
estimator. 

\begin{prop}\label{prop:DensityEstimatorConsistency}
Suppose $M$ is a compact $d$-dimensional smooth submanifold of $\bR^n$
with inherited volume form $\mu$.  Let $\rho$ be a density
supported on all of $M$.  Let $X$ be a collection of iid samples drawn
from $\rho$, and $\epsilon$ a bandwidth parameter.  Then for all $x
\in M$ that are points of continuity for $\rho$, the quantity
$\hat{\rho}_\epsilon(x)$ converges in probability to $\rho(x)$ as
$\card{X} \to \infty$, provided $\epsilon \to 0$ and $\epsilon
(\card{X})^{1/2d} \to \infty$.
\end{prop}

\begin{proof}
    The idea is that the law of large numbers implies \[\lim_{\card{X}
      \to \infty} \frac{\card({X \cap B_\epsilon(x)})}{\card{X}} = \int_{M
      \cap B_\epsilon(x)} \rho \, d\mu.\]  Let us make this more
    precise.  First, using the preceding lemma and the continuity of
    $\rho$ we have 
    \begin{equation*}
        \frac{1}{V_d \epsilon^d} \int_{M \cap B_\epsilon(x)} \rho \, d\mu = \frac{1 + O(\epsilon)}{\mu(M \cap B_\epsilon(x))} \int_{M \cap B_\epsilon(x)} \rho \, d\mu \xrightarrow{\epsilon \to 0} \rho(x).
    \end{equation*}
    
    Hence for $t > 0$, we have
    \begin{align*}
        &\mathbb{P}(\abs{\hat{\rho}_\epsilon(x) - \rho(x)} \geq t) \\
        &\leq \mathbb{P}\left(\abs{\hat{\rho}_\epsilon(x) - \frac{1}{V_d \epsilon^d} \int_{M \cap B_\epsilon(x)} \rho \, d\mu} \geq t + O(\epsilon)\right) \\
        &= \mathbb{P}\left(\abs{\frac{\card({X \cap B_\epsilon(x)})}{\card{X}} -  \int_{M \cap B_\epsilon(x)} \rho \, d\mu} \geq V_d \epsilon^d (t + O(\epsilon))\right) \\
        &\leq 2 \exp\left(-2 (\card{X}) (V_d \epsilon^d (t + O(\epsilon)))^2\right)
    \end{align*}
    
    \noindent where the last line uses Hoeffding's inequality.  
    
    Consequently, if $\epsilon (\card{X})^{1/2d} \to \infty$, then
$\hat{\rho}_\epsilon(x)$ converges to $\rho(x)$ in probability as
$\card{X} \to \infty$ and $\epsilon \to 0$.
\end{proof}

\begin{rem}
To get a \emph{point} estimate of density, we
need $\epsilon$ to tend to zero, but the formula in the proposition
indicates that it cannot go to zero too quickly, and the rate at which
it can go to $0$ is controlled by the dimension $d$.  A larger $d$ forces a slower
approach to zero; this is an example of the so-called \emph{curse of
dimensionality}.
\end{rem}

\begin{rem}
    Recall that we only need an estimator that is proportional to density.  Therefore, we might as well use 
    \begin{equation}\label{eq:countdensityestimator}
        \hat{\rho}_\epsilon(x) = \card({X \cap B_\epsilon(x)})
    \end{equation}
    \noindent instead, i.e., simply count the number of points in an
$\epsilon$-neighborhood of $x \in \mathbb{R}^n$ without dividing by the total number of points and the volume
of a standard ball.  In particular, while the consistency of the estimator may depend on the intrinsic dimension $d$ of the underlying manifold, the estimator itself does not, up to global scaling.  
\end{rem}

\begin{rem}
We offer a heuristic for choosing the bandwidth parameter
$\epsilon$ for the density estimator.  According to the multivariate
\emph{Scott's rule} \cite{Scott2015}, we can set $$\epsilon =
\hat{\sigma} \cdot (\card{X})^{-1/(d+4)},$$ where $\hat{\sigma}^2 =
\det(\hat{\Sigma})^{1/d}$ is the geometric mean of the eigenvalues of
the positive-definite sample covariance matrix $\hat{\Sigma}$ of $X$.

For this rule to satisfy the conditions of Proposition
\ref{prop:DensityEstimatorConsistency}, we would
need $$(\card{X})^{1/(d+4) - 1/(2d)} \ll \hat{\sigma} \ll
(\card{X})^{1/(d+4)}$$ asymptotically.  In particular, if $d$ is not
sufficiently small, then surprisingly $\hat{\sigma}$ has to
\emph{grow} with $\card{X}$.  Hence, this heuristic is perhaps best
used when the intrinsic dimension $d$ is small.   

Alternatively, the value of the hyperparameter $\epsilon$ can be
manually tuned, for example by maximizing the mutual information
metric we describe in Section \ref{ss:mi}.
\end{rem}

The above discussion provides the theoretical justification for the
following standard rejection sampling procedure to generate an
ensemble of subsamples on which to produce circular coordinates.

\startcode
\textbf{Input:} A data set $X\subset \bR^{\ambdim}$.

\textbf{Step 1:} For each point $x\in X$, compute $\hat\rho_{\epsilon}(x)$
by the formula~\eqref{eq:countdensityestimator}.  We define $\pi\colon X\to [0,1]$ to be
$x \mapsto m/\hat\rho_{\epsilon}(x)$ where $m \leq \min_x \hat{\rho}_\epsilon(x)$ is chosen to guarantee
a given expected subsample size.

\textbf{Step 2:} For $i=1,\dotsc,k$, create $X_{i}$ as follows.
For each point $z \in X$, draw a uniform random number $\tau_{i,z} \in
[0,1]$ and accept $z$ in $X_{i}$ if and only if $\tau_{i,z} \leq
\pi(x)$.

\textbf{Step 3:} Return subsamples $X_{1},\dotsc,X_{k}$.

\stopcode

\section{Averaging coordinates with Procrustes alignment and hill climbing}
\label{sec:procrustes}

The purpose of this section is to explain an algorithm for averaging
circular coordinates produced on subsamples of data.  We do this
by adapting the ``align and average'' framework for manifold learning
from our previous paper~\cite{BCFM}.  The basic outline is that we use the
Procrustes problem to align the data sets and then take a suitable
mean.  We begin by reviewing the Procrustes problem specialized to the
circle.

\begin{defn}\label{defn:procrustes}
The \emph{generalized Procrustes problem on the circle} is the following: given
a set of $k$ input configurations of $n$ ordered points on $S^{1}$, each represented as a function $\Phi_{i} : \{1, \ldots, n\} \to S^1$, determine the
optimal rotations and reflections $g_{1},\dotsc,g_{k}$ and centroid configuration $\Theta$ that
minimizes the loss functional  
    \begin{equation*}
        L((g_{1},\dotsc,g_{k}),\Theta) = 
	\frac{1}{k} \sum_{i=1}^{k}\sum_{j=1}^{n} d_{S^1}(g_i \cdot \Phi_i(j),\Theta(j))^2.
    \end{equation*}
\end{defn}

Algorithms to find approximate solutions to the corresponding problem
in Euclidean space go back more than 50 years, and in~\cite{BCFM} the
authors used such algorithms to average dimensionality-reducing
embeddings on subsamples.  In contrast to the current paper,
\cite{BCFM} did not perform density equalization and assumed large
overlap between subsamples; the Procrustes solution algorithms extend
to this ``missing values'' case.  In this paper,
subsamples are not expected to have large overlap, and Procrustes
solution algorithms do not readily adapt to this context:
when data sets do not overlap, it does not make sense to ask for an
alignment.  Therefore, the problem and approach in this paper differ from
\cite{BCFM} in two fundamental ways: first, we study circular
coordinates instead of embeddings into Euclidean space,
and second, instead of approximating the solution to the missing
values Procrustes problem on the subsample coordinates, we first extend the
subsample coordinates to the whole data set and then find the
approximate solution to the Procrustes problem on the circle
(Definition~\ref{defn:procrustes}).

Assume we are given subsamples $X_{i}\subset X$, $i=1,\dotsc,k$ each with
coordinates $\phi_i\colon X_i \to S^{1}$.  We assume that $\card{X_i}
\ll \card{X}$ and we do not necessarily assume that any given $x \in X$
is contained in many subsamples.  

Our first step is to extend each
coordinate $\phi_i \colon X_i \to S^{1}$ to all of $X$.  
We extend $\phi_i$ to a circular coordinate $\tilde\phi_{i}\colon X\to S^{1}$ on all of $X$
by a weighted average of $\phi_i$ evaluated at points of $X_i$, using the
\emph{circular mean} for averaging.  Modeling $S^{1}$ as
the unit circle in $\bR^{2}$, we set
\begin{equation}\label{eq:circmean}
\tilde\phi_{i}(x) = \circmean(w,X_{i},\phi_{i},x):=
\left.{\displaystyle \sum_{y \in X_i} w(x, y) \phi_{i}(y)}\middle/
\biggl\Vert \displaystyle \sum_{y \in X_i} w(x, y) \phi_{i}(y) \biggr\Vert\right.,
\end{equation}
for $x\notin X_{i}$. Here $w(x, y)$ is a weighting function; a
standard choice is the Gaussian kernel $w(x,y) = e^{-\beta \norm{x - y}^2}$ for 
some $\beta > 0$.  The circular mean is undefined when the sum
inside the formula is zero and is unstable for values near
zero.  However, for a reasonable initial embedding $\phi_i$, points in
$X_i$ near any given point $x$ land in a small interval of the circle
and $\beta$ could be chosen so that the resulting $\tilde\phi_{i}(x)$ remains
in that interval.  In practice, we took $\beta \approx \epsilon$,
where $\epsilon$ is the scale parameter used in the density estimation
procedure.  In mild abuse
of notation, we continue to denote these extensions by $\phi_i \colon
X \to S^1$. 

The next step is to find an approximate solution for the Procrustes
problem on the circle for the images of the extended embeddings.  This
starts by leveraging solvers for the Euclidean Procrustes problem
to approximate solutions in the circular case.  Specifically, viewing
$S^1$ as a subset of $\bR^2$, we solve the corresponding
Procrustes problem for point clouds in $\bR^{2}$, where the allowed
transformations are restricted to rotations centered on the origin and
reflections through a line through the origin, that is, elements of
$O(2)$.  Although these transformations send the unit circle to
itself, and restricted to the unit circle are the same transformations
we consider in the Procrustes problem on the circle
(Definition~\ref{defn:procrustes}), a solution to the $O(2)$ Procrustes 
problem in $\bR^{2}$ does not necessarily give a solution to the
circle version because (1) the points in the
centroid found lie within but not necessarily on the circle, and (2) the metric on the circle
in the loss function $L$ of Definition~\ref{defn:procrustes}
is the natural arclength metric and not the 
restriction of the metric on $\bR^{2}$. 

We
address the first issue by forcing the centroid $\Theta$ to lie on the circle
by projecting back.  Although numerically unlikely, it is possible
that the origin may appear as one of the points in
the $\bR^{2}$ centroid, and in this case it is a least squares minimization
problem to choose the point for $\Theta$ to be the one that minimizes
the loss function.  For the second issue, we have the following
observation bounding the distances on the circle and distances in
$\bR^{2}$ for a point on the circle and a point inside the circle
projected back to the circle.

\begin{lem}
Let $x,y$ be elements of the unit circle in $\bR^{2}$ and let $r\in [0,1]$, then 
\[
d_{S^{1}}(x,y)\leq \pi d_{\bR^{2}}(x,ry)
\]
\end{lem}

As a consequence, we prove that a solution to the $O(2)$ Procrustes problem in
$\bR^{2}$ for configurations on $S^{1}$ will be close to a solution to
the Procrustes problem on the circle, depending on the optimal value of the loss function $L$ of
Definition~\ref{defn:procrustes}.

\begin{prop}
Given $k$ configurations of $n$ points on the circle
$\Phi_{1},\dotsc,\Phi_{k}$, 
let $((g_{1},\dotsc,g_{k}),\Theta)$ be the solution to the generalized
Procrustes problem on the circle, let
$((g'_{1},\dotsc,g'_{k}),\Psi)$ be the solution to the $O(2)$
Procrustes problem on $\bR^{2}$, and let $\Theta'$ be the projection
of $\Psi$ back to the circle as
described above.  Then 
\[ \inf_{h\in O(2)}
d(h\cdot \Theta,\Theta')\leq (1 + \pi) \sqrt{L^*},
\]
\noindent where $L^* = L((g_{1},\dotsc,g_{k}),\Theta)$ is the minimum value of the generalized Procrustes loss on the circle.
\end{prop}

\begin{proof}
For clarity, we subscript $d$ with the symbol for the metric space to
denote the respective distances.   By the triangle inequality and
the proposition above, we have
\begin{align*}
\inf_{h\in O(2)}d_{(S^{1})^{n}}(h\cdot \Theta,\Theta')
&=\inf_{h,h'\in O(2)}d_{(S^{1})^{n}}(h\cdot \Theta,h'\cdot \Theta')\\
&\leq \inf_{h,h'}(d_{(S^{1})^{n}}(h\cdot \Theta,\Phi_{i})
   +d_{(S^{1})^{n}}(h'\cdot \Theta',\Phi_{i}))\\
&= \inf_{h,h'}(d_{(S^{1})^{n}}(h \cdot \Phi_{i},\Theta)
   + d_{(S^{1})^{n}}(h'\cdot \Phi_{i},\Theta'))\\
&= \inf_{h} d_{(S^{1})^{n}}(h \cdot \Phi_{i},\Theta)
   +\inf_{h} d_{(S^{1})^{n}}(h \cdot \Phi_{i},\Theta')\\
&\leq \inf_{h} d_{(S^{1})^{n}}(h\cdot \Phi_{i},\Theta)
   +\inf_{h} \pi d_{(\bR^{2})^{n}}(h\cdot \Phi_{i},\Psi),
\end{align*}
where we use the $\ell^{2}$ distances on the product metric spaces.
Looking at these as coordinates for elements of
$\bR^{k}$ (as $i$ varies from 1 to $k$), and applying the triangle
inequality in $\bR^{k}$, we get 
\begin{equation}\label{eq:Pest}
\begin{aligned}
&\inf_{h}d_{(S^{1})^{n}}(h\cdot \Theta,\Theta')
=\frac1{\sqrt{k}}\left(\sum_{i=1}^{k}
\inf_{h}d_{(S^{1})^{n}}(h\cdot \Theta,\Theta')^{2}
\right)^{1/2}\\
&\qquad\leq 
\frac1{\sqrt{k}}\left(\sum_{i=1}^{k} (\inf_{h}d_{(S^{1})^{n}}(h\cdot \Phi_{i},\Theta)
   +\inf_{h} \pi d_{(\bR^{2})^{n}}(h\cdot \Phi_{i},\Psi))^{2}
\right)^{1/2}\\
&\qquad\leq \frac1{\sqrt{k}}\left(\sum_{i=1}^{k} \inf_{h}d_{(S^{1})^{n}}(h\cdot \Phi_{i},\Theta)^{2}
\right)^{1/2} + \frac\pi{\sqrt{k}}\left( \sum_{i=1}^{k}
   \inf_{h} d_{(\bR^{2})^{n}}(h\cdot \Phi_{i},\Psi)^{2}
\right)^{1/2}\\
&\qquad= \left(\frac1k\sum_{i=1}^{k} \inf_{h}d_{(S^{1})^{n}}(h\cdot \Phi_{i},\Theta)^{2}
\right)^{1/2} + \pi\left( \frac1k \sum_{i=1}^{k}
   \inf_{h} d_{(\bR^{2})^{n}}(h\cdot \Phi_{i},\Psi)^{2}
\right)^{1/2}
\end{aligned}
\end{equation}
The first term on the last line is then
$\sqrt{L^*} = (L((g_{1},\dotsc,g_{k}),\Theta))^{1/2}$, since
by hypothesis $((g_{1},\dotsc,g_{k}),\Theta)$ is the solution to the generalized
Procrustes problem on the circle for the configurations
$\Phi_{1},\dotsc,\Phi_{k}$.  For the second term, write
$L_{\bR^{2}}$ for the Procrustes loss functional in~$\bR^{2}$,
\[
L_{\bR^{2}}((h_{1},\dotsc,h_{k}),\Xi):= \frac{1}{k} \sum_{i=1}^{k}\sum_{j=1}^{n}
d_{\bR^{2}}(h_{i}\cdot \Phi_{i}(j),\Xi(j))^{2}
=\frac{1}{k} \sum_{i=1}^{k} d_{(\bR^{2})^{n}}(h_{i}\cdot \Phi_{i},\Xi)^{2}.
\]
Since 
$((g'_{1},\dotsc,g'_{k}),\Psi)$ is the solution to the $O(2)$ Procrustes
problem in $\bR^{2}$ for the configurations
$\Phi_{1},\dotsc,\Phi_{k}$, we have that 
\[
\frac{1}{k} \sum_{i=1}^{k}
   \inf_{h} d_{(\bR^{2})^{n}}(h\cdot \Phi_{i},\Psi)^{2}=
L_{\bR^{2}}((g'_{1},\dotsc,g'_{k}),\Psi)\leq
L_{\bR^{2}}((g_{1},\dotsc,g_{k}),\Theta)
\]
and since the standard and Euclidean metrics on the circle satisfy
$d_{\bR^{2}}\leq d_{S^{1}}$, we have
\[
L_{\bR^{2}}((g_{1},\dotsc,g_{k}),\Theta)\leq
L((g_{1},\dotsc,g_{k}),\Theta) = L^*.
\]
Plugging these inequalities back into~\eqref{eq:Pest}, we get
the inequality in the statement. 
\end{proof}

The purpose of the approximate solution using the $O(2)$ Procrustes
problem in $\bR^{2}$ is that solvers for this problem use non-local
methods, which can jump out of basins of attraction to which local methods like
hill climbing are confined.  This is key because $O(2)^k$ is not connected and the error surface of the loss function is not convex
within connected regions.  However, once we have an approximate solution with an
exact solution nearby, hill climbing methods then can effectively zero
in on the exact solution.  

Our algorithm starts with an approximate
solution from the $O(2)$ Procrustes problem in $\bR^{2}$, constructed
for example using the algorithm of ten Berge~\cite[p. 272]{TenBerge1977}
as a seed position to start hill climbing.
Then, we hill climb by iteratively rotating according to a learning rate
schedule (i.e., a decreasing function of the iteration number) each
of the input configurations as well as each point in the candidate
centroid, choosing among the clockwise, counterclockwise, or no rotation for
each configuration and each point of the candidate centroid giving
the smallest value for the loss function.  We terminate and return the
candidate centroid when some convergence criterion is met (e.g., the
change in the loss function falls below a pre-defined threshold).

We summarize our algorithm as follows:  

\startcode
\textbf{Input:} A data set $X\subset \bR^{\ambdim}$, subsamples $X_{i}\subset
X$, and circular coordinates\break $\phi_{i}\colon X_{i}\to S^{1}$ for
$i=1,\dotsc,k$.

\textbf{Step 1:} Extend each $\phi_{i}$ to a function $\Phi_{i}\colon
X\to S^{1}$ using the circular mean: for $x\in X_{i}$ define
$\Phi_{i}(x)=\phi_{i}(x)$ and for $x\notin X_{i}$, define
$\Phi_i(x)=\circmean(w,X_{i},\phi_{i},x)$ (see~\eqref{eq:circmean} and
text following for description). 

\textbf{Step 2:} Solve the $O(2)$ Procrustes problem in $\bR^2$ for
the configurations\break $\Phi_{1},\dotsc,\Phi_{k}$ via the 
ten Berge algorithm to produce rotation/reflection guesses
$g_{1},\dotsc,g_{k}$ and a centroid configuration $\Theta$
obtained by radially projecting the centroid obtained in $\bR^{2}$ back to the
circle.

\textbf{Step 3:} Use hill climbing starting at these guesses to search
for local minima for the loss function $L$ in
Definition~\ref{defn:procrustes}. 

\textbf{Step 4:} Return the final centroid configuration.

\stopcode

\section{Experimental results}\label{sec:experimental}

We undertook several experiments with both synthetic and real-world
data to test the effectiveness of the algorithm.  In
subsection~\ref{ss:density}, we illustrate with synthetic data what
goes wrong with circular coordinates derived from unevenly sampled
data and observe that the algorithm corrects the coordinates.  In
subsection~\ref{ss:worm}, we review two sets of real-world data on
neuronal recordings from \textit{C. elegans}.  In the first set of these results, we
used our algorithm to re-analyze the dataset from
Kato et al.~\cite{Kato2017} on the relationship of locomotory behavior to
recurrent trajectories in neuronal phase space.  In the second set, we
analyzed recent data from Yemini et al.~\cite{Yemini2021} on the response
of worms to various sensory stimuli.  

For each experiment, we computed and compared two coordinates: the global ``uncorrected coordinate'' obtained by calculating a single circular coordinate on
the data without subsampling and the ``corrected coordinate'' obtained from our proposed approach of subsampling, aligning, and
averaging coordinates.

\subsection{Unbalanced circles and ellipses}\label{ss:density}

To illustrate what can go wrong with circular coordinates for an
uneven density, we generated a synthetic unbalanced circle data set as
follows.  The 
data set consists of 1000 independently generated points in $\bR^{2}$, where the
distance from the origin to each point is normally distributed with
mean $1$ and standard deviation $0.1$ and the angles that these
vectors make with the positive $x$-axis are drawn from a von Mises
distribution with dispersion $1.3$.  This produces samples concentrated around the unit circle with a denser region
near the angle $0=2\pi$ radians. Producing circular coordinates
heavily overweights the denser region; see
Figure~\ref{fig:unbalanced_circle}.

\begin{figure}
    \centering
    \includegraphics[width=0.9\textwidth]{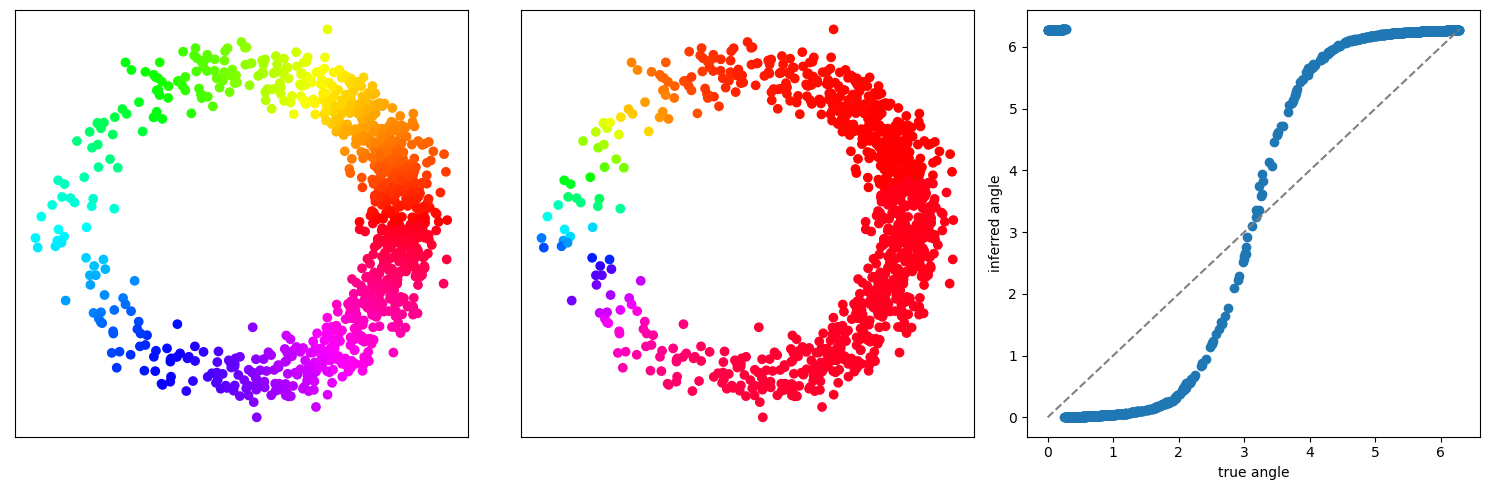} \\%
    \includegraphics[width=0.9\textwidth]{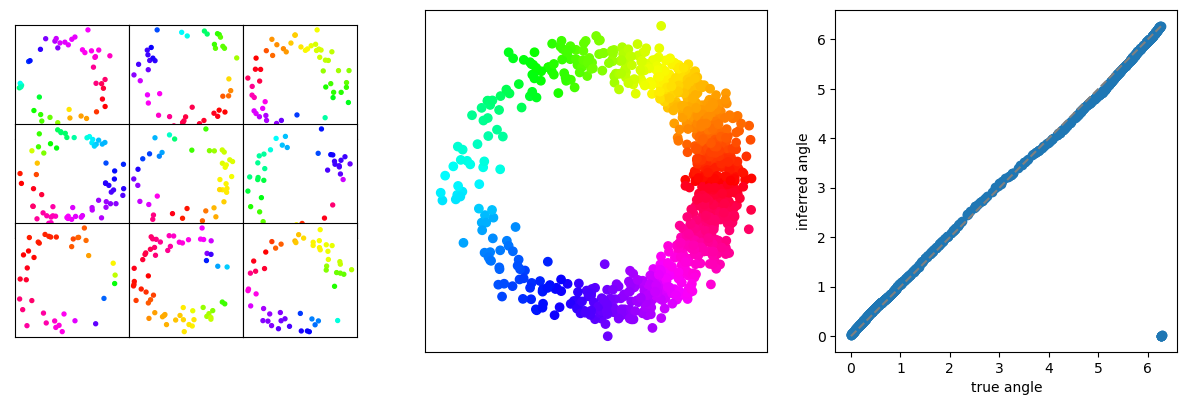}
    \caption{Top: An unbalanced circle dataset colored by the ``true'' angle (left) or the uncorrected coordinate inferred from persistent cohomology (center).  A plot of the uncorrected coordinate against the true coordinate (right); the dotted diagonal line indicates equality.  Bottom: Examples of the subsampled unbalanced circle and their inferred phases (left), the final corrected coordinate obtained by aligning the coordinates assigned to subsamples (center), and a plot of the corrected coordinate against the true coordinate (right).}
    \label{fig:unbalanced_circle}
\end{figure}

Density-based subsampling drastically improves the obtained coordinate
in this artificial example.  Using the count density estimator
$\hat{\rho}_\epsilon$ for $\epsilon$ chosen according to Scott's rule, we generated $30$
density-equalized subsamples with expected size $50$.  Then, we used
persistent cohomology to find circular coordinates and used
generalized orthogonal Procrustes alignment to obtain the corrected
coordinate.  The results are shown in the bottom row of
Figure~\ref{fig:unbalanced_circle}.  The agreement between the true
angle and 
the inferred angle is much better in this case.  

Next, we considered less symmetric examples where the reach of the
manifold becomes involved: unbalanced ellipses that are denser near one
vertex.  The procedure for generating these point clouds is the same
for the unbalanced circle, except that we applied a dilation factor of
$1.6$ in the $x$ direction; see the top row of
Figure~\ref{fig:unbalanced_ellipse}.  The inferred coordinate does not
recover the arc length exactly, but nonetheless we still see a
significant improvement using the corrected coordinates.  See the bottom row of Figure~\ref{fig:unbalanced_ellipse}. 
    
\begin{figure}
    \centering
    \includegraphics[width=0.9\textwidth]{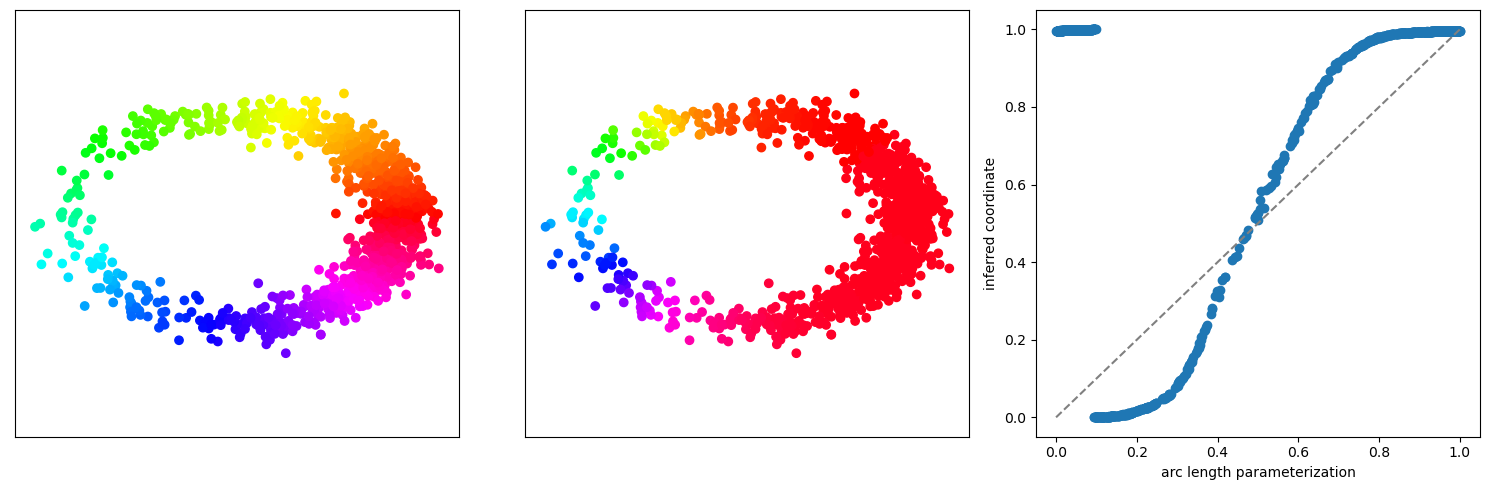} \\%
    \includegraphics[width=0.9\textwidth]{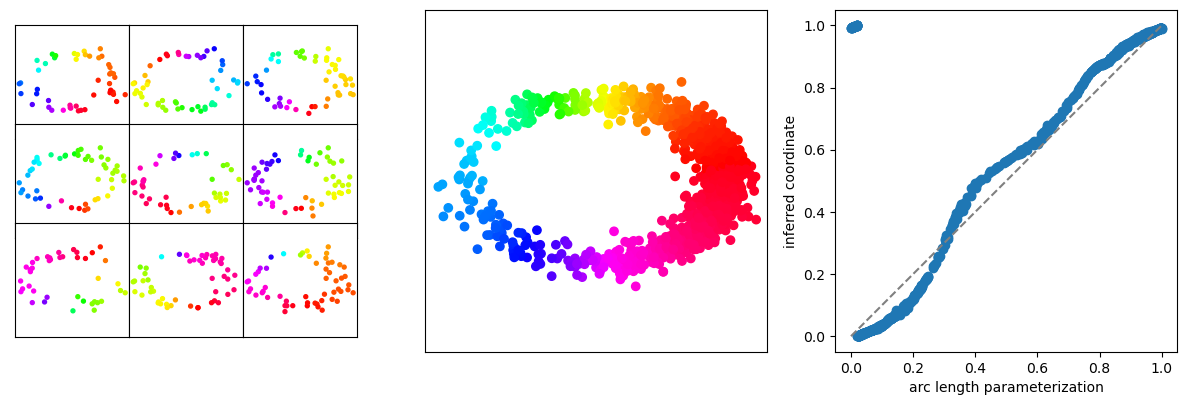}
    \caption{Top: An unbalanced ellipse dataset colored by arc length (left) or the uncorrected coordinate inferred from persistent cohomology (center).  A plot of the uncorrected coordinate against the arc length parametrization (right).  Bottom: Examples of the subsampled unbalanced ellipse and their inferred phases (left), the final corrected coordinate (center), and a plot of the corrected coordinate against the arc length parametrization (right).}
    \label{fig:unbalanced_ellipse}
\end{figure}

\subsection{\textit{C.\ elegans} neuronal recordings}
\label{ss:worm}

\subsubsection{Data from Kato et al.}
We used our method to construct coordinates on the neuronal activity phase
space of the nematode \textit{C.\ elegans}. The Kato et al.\ dataset~\cite{Kato2017} recorded multivariate time
series with features corresponding to the calcium traces of a subset of neurons sampled at a particular frequency.  The animals were
transgenically designed so that the neuron centers fluoresce upon
calcium influx, which is a proxy for neuronal activation.  The
original analysis found that the different regions
of the phase space extracted from this high-dimensional data reflect
the different states of the worm locomotory gait, termed
\emph{sinusoids} or \emph{pirouettes}.  The worm cycles through these
different states as it executes its motor commands; in particular, the
underlying dynamical system is recurrent and we expect the
trajectories to consist of loops in the phase space.
    
Several methods have been applied to study these non-stationary and
high-dimensional neuronal activity trajectories, including adaptive
locally linear segmentation \cite{Costa2019}, hierarchical recurrent
state space models \cite{Linderman2019}, asymmetric diffusion mapping
\cite{Brennan2019, Brennan2023}, dynamic mode decomposition with
control \cite{Fieseler2020}, and nonlinear control
\cite{Morrison2021}.  Beyond a general qualitative description, a
quantitative model of the different states of the worm neuronal manifold
-- continuous or discrete -- is a preliminary step to understand
inter-state transitions and to generate testable hypotheses about the
underlying biological mechanisms that drive and control the dynamical
system and the external factors that may alter it.   

As validation of our method, we applied it to recover these recurrent
patterns and more from the \textit{C. elegans} neuronal activity data in
a flexible manner and perform nonlinear continuous state space
segmentation.  TDA methods have been previously used to study the
activity of neuron ensembles, for example, in head
direction cells \cite{Rybakken2019} or in grid
cells~\cite{Gardner2022} in mice.  Our analysis demonstrates that these techniques are also applicable to
study whole-brain recordings (as opposed to specific groups of neurons
with spatial-processing function) of an invertebrate nervous system.
Furthermore, in this case the manifold of interest has no direct
physical interpretation like head direction or allocentric location,
but rather is an abstract representation of brain state.
    
We reproduced the dimensionally-reduced version of the phase space
from~\cite{Kato2017}.  More
specifically, by phase space here we mean a combination of all the
available activity traces and their derivatives obtained through
total-variation regularization.  Then, we examined the persistent cohomology results to determine the general global shape of the phase space.  For instance, for the worm neuronal trajectory depicted in Figure~\ref{fig:kato-phdgms}(a), we see that there is an obvious large loop that is confirmed by a long-lived class in $PH^1$ with late death time, which we deem suitable for further coordinatization.  On the other hand, for the worm neuronal trajectory shown in Figure~\ref{fig:kato-phdgms}(b), we do not detect a clear persistent large loop but rather multiple smaller loops, which suggests that the shape should be modeled differently: imposing a circular coordinate onto the data in its current form is unlikely to be effective.  This is not to say that this worm is definitely not exhibiting the expected cyclic locomotory brain patterns, but that additional preprocessing of the data is perhaps needed.  This shows one of the advantages of our TDA approach, which separates the
problem of shape detection from shape coordinatization.  While other methods
assume -- correctly or not -- prior knowledge about the general shape of neuronal
trajectories, we can first detect loops in the data \emph{prior} to finding coordinates for them (if they exist).

\begin{figure}
    \centering
    \begin{subfigure}[b]{0.7\textwidth}
        \centering
        \includegraphics[width=\textwidth]{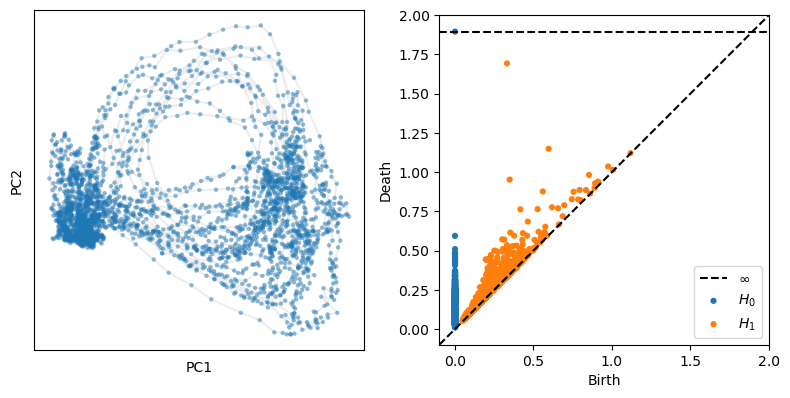}
        \caption{}
    \end{subfigure}\\
    \begin{subfigure}[b]{0.7\textwidth}
        \centering
        \includegraphics[width=\textwidth]{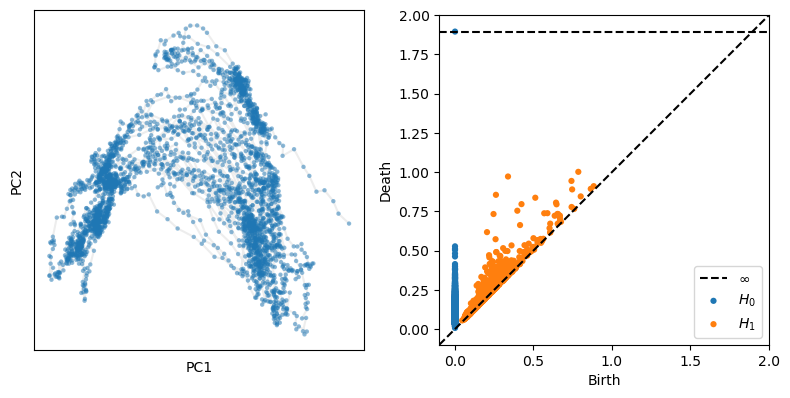}
        \caption{}
    \end{subfigure}
    \caption{The projection of the neuronal trajectories into PCA space (left) and the persistence diagram (right) for two different worms.}
    \label{fig:kato-phdgms}
\end{figure}

For the phase space of the first worm above with the long-lived cohomology class, we calculated the uncorrected and corrected circular coordinates; see Figure~\ref{fig:kato5}.
By visual inspection, the uncorrected coordinate
concentrates more of the phase changes near the top of the loop and
less near the bottom, whereas the corrected coordinate equalize the
size of the phases.  Note that the density imbalance that arises here
is not a mere sampling error but a natural real-life consequence of the brain state prior to the worm's decision to start turning.

\begin{figure}
    \centering
    \begin{subfigure}[b]{0.35\textwidth}
        \centering
        \includegraphics[width=\textwidth]{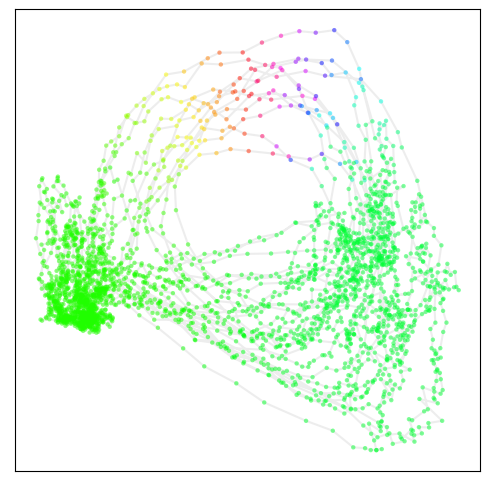}
        \caption{}
    \end{subfigure}
    \hspace{5mm}
    \begin{subfigure}[b]{0.35\textwidth}
        \centering
        \includegraphics[width=\textwidth]{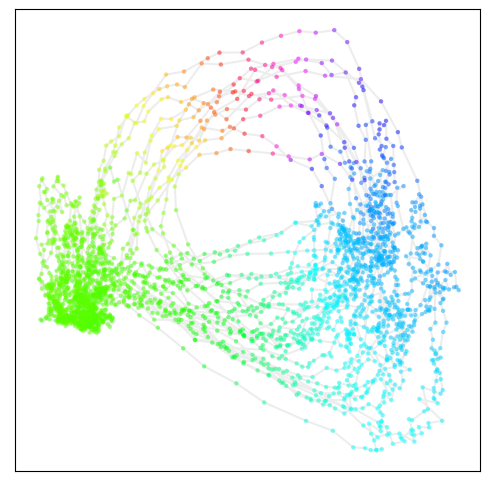}
        \caption{}
    \end{subfigure}
    \caption{The neuronal manifold of \textit{Caenorhabditis
elegans} global brain dynamics, coordinatizing a cyclic
locomotory gait, colored by (a) the uncorrected coordinate or (b)
the corrected coordinate.} 
    \label{fig:kato5}
\end{figure}

For this data set, we also have annotations into discrete brain states
provided in the original article, which were obtained by examining the
trace of a single command interneuron (AVAL) and manually setting
thresholds.  Figure~\ref{fig:kato5-comparison} compares the continuous
coordinates inferred using our approach with the original annotations
by plotting the coordinates in the tangent bundle $TS^1 \cong S^1
\times \bR$ as a polar plot; the tangent vectors were obtained from
numerically differentiating the coordinates.  Overall, we see
coordinate values from the same state are clustered together, which
demonstrates agreement between the new and old results.  This shows
that our algorithm recovers the model of the brain state manifold
proposed in \cite{Kato2017}, c.f.\ Figure 4(E) in that paper.  We also
see that the uncorrected coordinate tends to obscure the smaller loop
corresponding to ventral turns, whereas the corrected coordinate
manages to recover this second loop more clearly.
    
\begin{figure}
    \centering
    \begin{subfigure}[b]{0.35\textwidth}
        \centering
        \includegraphics[width=\textwidth]{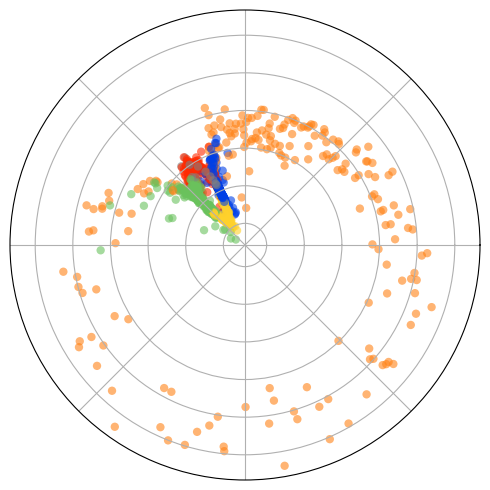}
        \caption{}
    \end{subfigure}
    \hspace{5mm}
    \begin{subfigure}[b]{0.35\textwidth}
        \centering
        \includegraphics[width=\textwidth]{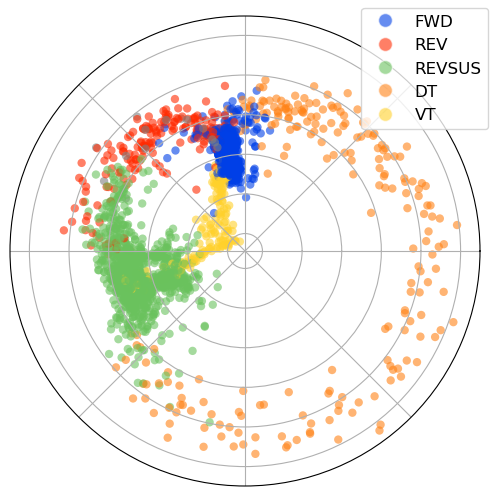}
        \caption{}
    \end{subfigure}   
    \caption{Comparison between the labels from \cite{Kato2017} and our
inferred coordinates, either (a) uncorrected or (b) corrected.  Points are colored according to the discrete
categories of the provided labels: forward (FWD), reversal (REV),
sustained reversal (REVSUS), dorsal turn (DT), and ventral turn
(VT).}
    \label{fig:kato5-comparison}
\end{figure}

\subsubsection{Data from Yemini et al.}
We also applied our methods to the more recent, more expansive
\textit{C.\ elegans} recordings obtained from Yemini et al.~\cite{Yemini2021}.
These are calcium trace recordings of approximately four minutes in
length per animal, sampled at 4 Hz.  In contrast to the recordings
analyzed in~\cite{Kato2017}, over a hundred neurons were recorded
simultaneously with individual resolution.  Furthermore, the animals
being recorded are subjected to various olfactory and gustatory
stimuli at minute intervals during the experiment.  The rapid
changes in the environment relative to the short durations of the
recordings complicate the analysis, but we nonetheless
show that the locomotory cycles can still be recovered using our methods.
    
To create a phase space from the raw data, we deployed techniques that
are commonly used to analyze time series from dynamical systems.  We took advantage of the fact that
the identities of the neurons in this experiment are individually
resolved and restricted to a few specific neurons of interest;
we focus on the command interneuron classes AVA and AVB whose
activities have been found to reflect various stages of the nematode
locomotory pattern.  These neurons are located downstream of most
sensory neurons in the olfactory and gustatory circuits and upstream
of motor neurons that control the muscles in the worm, and are known
to be hubs for sensorimotor integration.  First, we detrended the data
by removing mean and scale fluctuations against a moving average
background.  Second, we created a
delay embedding of the data, i.e., we mapped an observation function
$f\colon [0, T] \to \bR^\ambdim$ to its various lagged values: 
\[
S_{d, \tau} f\colon 
t \mapsto (f(t), f(t + \tau), \ldots, f(t + d\tau)).
\]  
The rationale for doing so comes from \emph{Takens' embedding theorem}
~\cite{Takens1981}, which roughly asserts that the topology of
an attractor of a dynamical system can be recovered from a suitable
delay embedding of an observation function.  For instance, if we take
$d = \tau = 1$, then $S_{1,1} f$ has the same span as the function
evaluated at the original times together with its first differences --
in other words, the phase space we are constructing is a
generalization of the phase space built using function values and
derivatives.  We took $d = 4$ and $\tau = 20$ frames (roughly 5
seconds), but our results are robust with respect to these choices.
Finally, as the Takens embedding is typically very high-dimensional,
we performed linear PCA to reduce the number of dimensions to 5.

We present our results in Figure \ref{fig:yemini_coords}, which
displays the phase spaces for six worms, colored by the uncorrected and corrected
coordinates.  The uncorrected coordinates are slightly noisier and
less even; see the quantitative discussion in the next section for
more measures of improvement.  We again recover
coordinates which correspond to observed worm behaviors.

\begin{figure}
    \centering
    \begin{subfigure}[b]{0.96\textwidth}
        \centering
        \includegraphics[width=\textwidth]{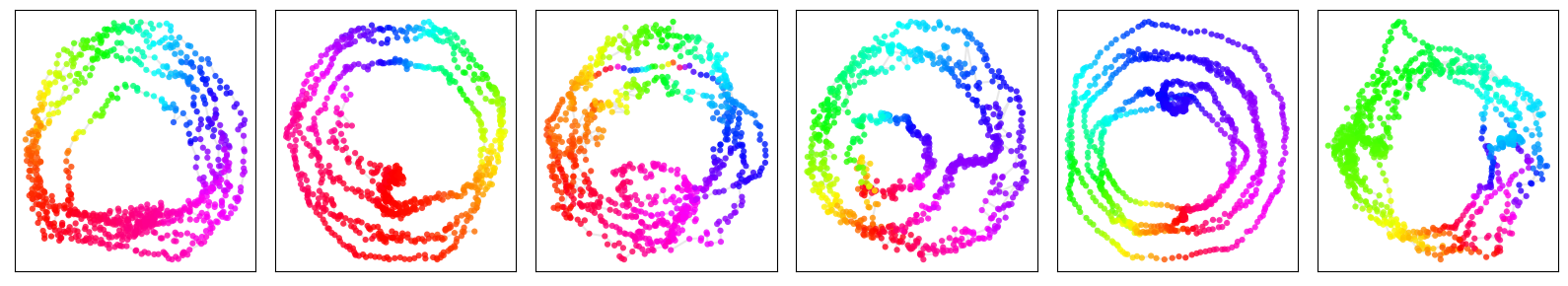}
        \caption{}
    \end{subfigure}
    \hspace{1cm}
    \begin{subfigure}[b]{0.96\textwidth}
        \centering
        \includegraphics[width=\textwidth]{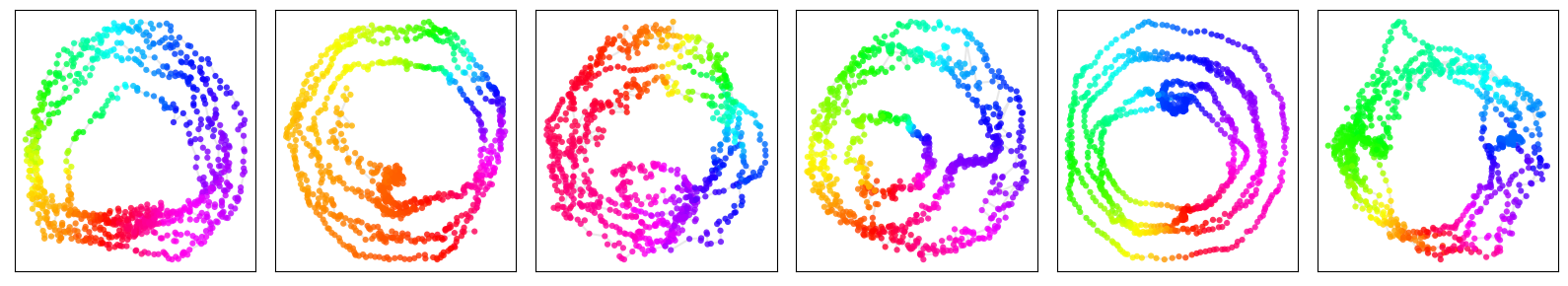}
        \caption{}
    \end{subfigure}
 
    \caption{\textit{C.\ elegans} neuronal manifolds constructed for six worms from \cite{Yemini2021}, colored by (a) the uncorrected coordinate and (b) the corrected coordinate.}
    \label{fig:yemini_coords}
\end{figure}

\subsection{Qualitative evaluation in terms of mutual information}
\label{ss:mi}

We propose an intrinsic criterion for evaluating the coordinates we
produce.  With real-world data sets, we do not typically have access
to ``ground truth''.  We could compare our results to expert-annotated
labels as we did for the Kato et al.\ dataset, but these are not always available, as in the case for the
Yemini et al.\ dataset.  Instead, we can quantify the association or correlation of a
coordinate with respect to the position in the full phase space using
\emph{mutual information}, which we briefly review.

Let $(X, Y)$ be a pair of continuous random variables with support
$\mathcal{X} \times \mathcal{Y}$, joint density $p_{(X,Y)}(x,y)$, and
marginal densities $p_X(x)$ and $p_Y(y)$.  Then the mutual information
between $X$ and $Y$ is defined as \[I(X, Y) = \int_{\mathcal{X} \times
  \mathcal{Y}} p_{(X,Y)}(x,y) \cdot \log \frac{p_{(X,Y)}(x,y)}{p_X(x)
  p_Y(y)} \, dx \, dy.\] Equivalently, mutual information can be
defined in terms of entropy: \[I(X, Y) = H(X) - H(X \mid Y) = H(Y) -
H(Y \mid X),\] i.e., $I(X,Y)$ is the reduction of uncertainty in $X$
once $Y$ is known, or vice versa.  If $X$ and $Y$ are independent, then $I(X,Y) = 0$.  Furthermore, if $(X,Y)$ is a bivariate Gaussian random variable, we have $I(X,Y) = -\frac{1}{2} \log(1 - \rho^2)$, where $\rho$ is the correlation between $X$ and $Y$.  This shows that mutual information is a form of nonlinear correlation between random variables.  
    
Kraskov, St\"ogbauer, and Grassberger~\cite{Kraskov2004} proposed a method to estimate $I(X, Y)$ from $N$ iid samples of $(X, Y)$ based on nearest neighbors.  Fix a positive integer $k$ (we set $k
= 3$ below), and consider the sup-metric on $\mathcal{X} \times
\mathcal{Y}$.  For each observation $(x,y)$, consider the smallest
rectangle centered at $(x,y)$ that contains the $k$ nearest neighbors
of $(x,y)$ according to the sup-metric; let the dimensions of this
rectangle be $\epsilon_x$ and $\epsilon_y$.  Let $n_x$ be the number
of other points $(x',y')$ with $d(x,x') \leq \frac{\epsilon_x}{2}$,
and let $n_y$ be the number of other points $(x',y')$ with $d(y,y')
\leq \frac{\epsilon_y}{2}$.  Then the \emph{KSG estimator} of the mutual
information is given by \[\hat{I}(X,Y) = \psi(k) - \frac{1}{k} - \frac{1}{N}
\sum_{(x,y)} (\psi(n_x) + \psi(n_y)) + \psi(N).\]  

While in principle the mutual information $I(X,Y)$ can be unbounded
above, the KSG estimator is bounded above: $$\hat{I}(X,Y) \leq
\hat{I}_{\text{max}} = \psi(N) - \psi(k) - \frac{1}{k}$$ since $n_x,
n_y \geq k$.  Equality is attained when $n_x = n_y = k$, i.e., the
$k$-nearest neighbors of a point are exactly the same whether they are
computed in $X$ or in $Y$.  This gives an additional reason why
$\hat{I}$ can be used to evaluate the concordance of two variables.
In order to standardize the values of the estimator $\hat{I}$ across
datasets, we normalize it by its maximum value and
calculate $$\hat{I}_{\text{norm}}(X,Y) =
\frac{\hat{I}(X,Y)}{\hat{I}_{\text{max}}}$$ instead.  Note that there
are several other definitions of normalized mutual information in the
literature, but $\hat{I}_{\text{norm}}$ is \emph{not} one of them: it
is simply a convenient normalization specific to the KSG estimator.
For more information on other normalized mutual information
estimators, see~\cite{Nagel2024}. 

Observe that the KSG estimator depends only on the metrics on $X$ and
$Y$.  Returning to our experiments, let $X$ be the dataset equipped with the Euclidean metric,
and let $Y$ be the same dataset, but with the metric induced by a
circular coordinate.  Different circular coordinates give rise to
different mutual information values.  The data processing inequality
implies that information is necessarily lost when we perform
dimensionality reduction, but a higher mutual information means that
more of the structure of the original dataset is retained by the
coordinate. 

To validate the use of KSG estimator to evaluate various circular
coordinates, we first applied it to the setup of the synthetic circle data
from above.  We generated 20 replicates of the unbalanced circle (with
the same parameters), extracted the uncorrected circular coordinate and
the corrected circular coordinate for each, and computed their mutual
information with the true coordinate.  The results are shown in
Figure~\ref{fig:mi}(a), where we see that the mutual information of
the corrected circular coordinate is significantly higher than the uncorrected
circular coordinate.  

Next, we turn to the worm neuronal activity data from Kato et
al.\ (Figure~\ref{fig:mi}(b)) and Yemini et al.\ (Figure~\ref{fig:mi}(c)).  While the mutual information estimates are now further from
their maximum possible value, we nonetheless see that the corrected coordinate from our approach
yields higher mutual information values than the uncorrected
coordinate in every case.  This demonstrates that our
subsample, align, and average method produces more informative coordinates and
captures more of the geometry of the original data than using
uncorrected coordinates in real-world data.

\begin{figure}
    \centering
    \begin{subfigure}[b]{0.32\textwidth}
        \centering
        \includegraphics[width=\textwidth]{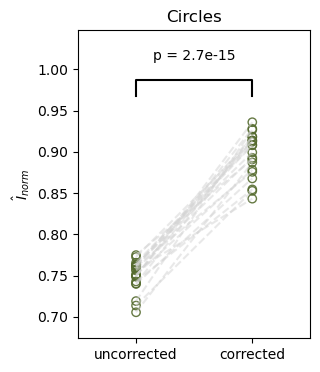}
        \caption{}
    \end{subfigure}
    \begin{subfigure}[b]{0.32\textwidth}
        \centering
        \includegraphics[width=\textwidth]{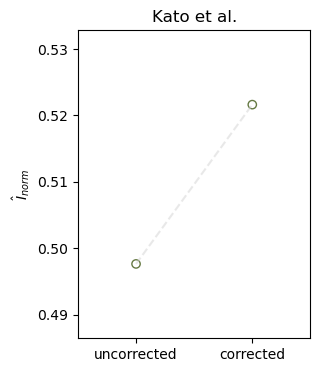}
        \caption{}
    \end{subfigure}
    \begin{subfigure}[b]{0.32\textwidth}
        \centering
        \includegraphics[width=\textwidth]{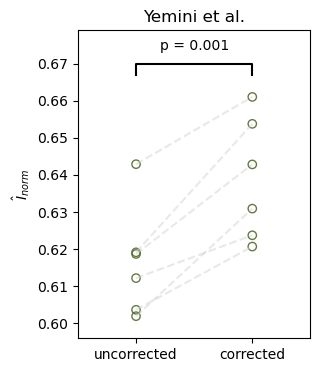}
        \caption{}
    \end{subfigure}
 
    \caption{Estimates of mutual information for (a) the unbalanced circle, (b) the worm from Kato et al., and (c) the six worms from Yemini et al.  Where displayed, the $p$-values are calculated from an one-sided paired $t$-test.}
    \label{fig:mi}
\end{figure}

\subsection{Runtime analysis}

Finally, we briefly demonstrate the time savings of our algorithm.  Basically, computing persistent cohomology and extracting
circular coordinates amounts to matrix manipulations which have
superlinear time complexity in the number of points (e.g.,
see~\cite{MMS11, MS24} which shows that persistent (co)homology and (co)cycles can be
computed in matrix multiplication time, asymptotically).  Therefore, by splitting the dataset into smaller pieces and then
computing and averaging persistent cohomology coordinates on the
subsamples, we may be able to obtain considerable decreases in
runtime.  See also the recent work~\cite{NM24} which explains how to
efficiently parallelize the distributed computation of persistent cohomology.

To demonstrate this in practice, we timed the execution of the
direct approach of computing the uncorrected and the corrected coordinate on two of the
datasets described above: the synthetic circle dataset (1000 points)
and the Kato et al. \textit{C. elegans} neuronal activity dataset
(3021 time points).  Each of these timing experiments were repeated 20 times,
and the results are displayed in Figure \ref{fig:timings}.   

\begin{figure}
    \centering
    \begin{subfigure}[b]{0.35\textwidth}
        \centering
        \includegraphics[height=5cm]{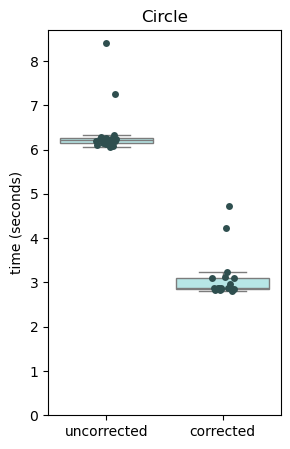}
        \caption{}
    \end{subfigure}
    \begin{subfigure}[b]{0.35\textwidth}
        \centering
        \includegraphics[height=5cm]{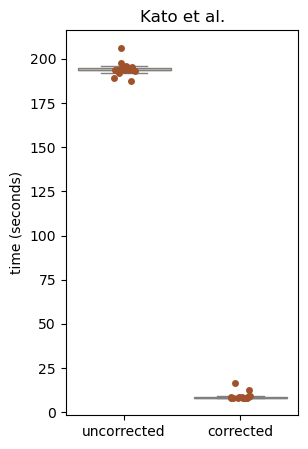}
        \caption{}
    \end{subfigure}
 
    \caption{Time taken to compute the uncorrected and corrected
      circular coordinates respectively, with 20 replicates in two
      datasets.}
    \label{fig:timings}
\end{figure}

For the circle dataset, the minimum runtime over 20 repeats for
computing the corrected circular coordinate is 2.80 seconds, compared
with 6.05 seconds for the uncorrected circular 
    coordinate, which is a $2.16\times$ runtime reduction.  For the
    worm dataset, the minimum runtime for the corrected circular
    coordinate is 8.17 seconds, compared with 188 seconds for the
    uncorrected circular coordinate, which is a $23.0\times$
    reduction.  This indicates that not only is the subsampling
    approach more time-efficient even with the extra alignment steps but also the time
    savings scales with the size of the datasets, which is important
    as the datasets that are generated nowadays are often very large.
    Our approach is well-equipped to handle these big datasets;
    indeed, by processing the data in smaller chunks we may be able to
    analyze datasets that could not be analyzed entirely within a
    single run at all.
    
\section{Discussion}

In this paper, we introduced a new algorithm for finding
robust circular coordinates on data that is expected to exhibit
recurrence.  Our algorithm has three main steps: first, it uses rejection 
sampling to correct for inhomogeneous sampling; second, it extracts the coordinates from
persistent cohomology; and finally, it applies
Procrustes matching to align and average the coordinates.  In experimental and synthetic data, this
subsampling and averaging approach produced coordinates which are
robust to noise and outliers and is significantly more efficient than
simply applying persistent cohomology to the entire data set.  Our
method is also stable by design, which is an essential desideratum for unsupervised data analysis.

There are a number of theoretical questions we intend to explore in
future work.  For one thing, the success of our algorithm raises
questions about geometries other than the circle for capturing
behavior of dynamical systems.  The circle is an obvious place to
start both because of its simplicity as well as its direct connection
to dimension 1 cohomology, but there are many other simple
model manifolds that could be worth exploring~\cite{Perea2018,Schonsheck2024}.  For another, there
remain many qualitative and quantitative questions about robustness
properties and the precise way our algorithm interacts with noise.
Finally, we do not understand the local convexity and global
properties of the optimization problem that arises when solving for the
Procrustes alignment.  Although in practice our algorithm converges
rapidly and is stable with respect to random jitter, it would be
desirable to have formal guarantees.

There are also many interesting scientific questions that arise from
the experimental studies.  The analysis of neuronal activity
recordings revealed a topological model of neuronal 
trajectories for \textit{C. elegans} that is
constructed from loops in which different regions of the brain state
space can be mapped to specific and interpretable macroscopic
behaviors in the worm.  Many questions arise about these trajectories
on the neuronal manifold:
    
\begin{itemize}

\item How universal or conserved are neuronal activity patterns -- in what ways does it vary between individual animals, or across various behavioral modes or environmental stimuli?

\item To what extent are future brain states predictable starting from an initial state?

\item What are the neuronal units and circuits that support the activity patterns we observe?

\end{itemize}

Using circular coordinates inferred from \textit{C.\ elegans} data,
our results confirm the locomotory cycle is highly conserved across
subjects and predictable to the extent that neuronal activity is
recurrent.  While this may be obvious in our simplified model system, for more complicated and higher-dimensional circuits it is
essential to have an unsupervised framework that can
operate in the high-dimensional ambient spaces of simultaneous
recordings of thousands of neurons or millions of voxels.

A next step would be to explore how
the presentation of different stimuli is picked up by the sensory
apparatus and integrated by decision-making circuits that output motor
commands, thereby changing the normal progression through the cycle.
Separately, given the prominence of the circular trajectories, it
would also be interesting to search for other phenomena and novel
patterns that are orthogonal to them in the neuronal activity of
\textit{C.\ elegans} \cite{Baas2017}.  In the data we analyzed this
neuronal circuit is more or less known since we are studying recurrent locomotory
behavior; in general one can search for neural subunits that
generate interesting coordinates.

Finally, all of these questions have analogues in many other neural
systems; for example, it would be interesting to understand whether
qualitatively similar conservation and decision-making proccesses
occur in animals such as flies and mice.

\FloatBarrier

\printbibliography

\end{document}